\documentclass{article}

\usepackage{PRIMEarxiv}

\usepackage[utf8]{inputenc} 
\usepackage[T1]{fontenc}    
\usepackage{hyperref}       
\usepackage{url}            
\usepackage{booktabs}       
\usepackage{nicefrac}       
\usepackage{microtype}      
\usepackage{lipsum}
\usepackage{fancyhdr}       

\usepackage{cite}
\usepackage{amsmath,amssymb,amsfonts}
\usepackage{algorithmic}
\usepackage{graphicx}
\usepackage{textcomp}
\usepackage{amsthm}
\theoremstyle{definition}
\newtheorem{definition}{Definition}
\newtheorem{theorem}{Theorem}
\theoremstyle{remark}

\usepackage{listings}
\usepackage{algorithm}
\usepackage{threeparttable}

\title{A Lightweight Modular Framework for Constructing Autonomous Agents Driven by Large Language Models: Design, Implementation, and Applications in AgentForge
\thanks{\textit{{This work is submitted for review to IEEE Access.}}} 
}

\author{
  A. A. Jafari, C. Ozcinar\\
  University of Tartu \\
  Tartu, Estonia \\
  \texttt{\{akbar.anbar.jafari\}@ut.ee} \\
   \And
  G. Anbarjafari \\
  3S Holding OÜ \\
  Tartu, Estonia\\
  \texttt{shb@3sholding.com} \\
}

\begin{document}
\maketitle

\begin{abstract}
The emergence of large language models (LLMs) has catalyzed a paradigm shift in autonomous agent development, enabling systems capable of reasoning, planning, and executing complex multi-step tasks. However, existing agent frameworks often suffer from architectural rigidity, vendor lock-in, and prohibitive complexity that impedes rapid prototyping and deployment. This paper presents AgentForge, a lightweight, open-source Python framework designed to democratize the construction of LLM-driven autonomous agents through a principled modular architecture. AgentForge introduces three key innovations: (1) a composable skill abstraction that enables fine-grained task decomposition with formally defined input-output contracts, (2) a unified LLM backend interface supporting seamless switching between cloud-based APIs (OpenAI, Groq) and local inference engines (HuggingFace Transformers), and (3) a declarative YAML-based configuration system that separates agent logic from implementation details. We formalize the skill composition mechanism as a directed acyclic graph (DAG) and prove its expressiveness for representing arbitrary sequential and parallel task workflows. Comprehensive experimental evaluation across four benchmark scenarios demonstrates that AgentForge achieves competitive task completion rates (87.3\% on web scraping pipelines, 91.2\% on data analysis tasks) while reducing development time by 62\% compared to LangChain and 78\% compared to direct API integration. Latency measurements confirm sub-100ms orchestration overhead, rendering the framework suitable for real-time applications. The modular design facilitates extension: we demonstrate the integration of six built-in skills (web scraping, data analysis, content generation, RSS monitoring, image generation, and voice synthesis) and provide comprehensive documentation for custom skill development. AgentForge addresses a critical gap in the LLM agent ecosystem by providing researchers and practitioners with a production-ready foundation for constructing, evaluating, and deploying autonomous agents without sacrificing flexibility or performance.
\end{abstract}


\section{Introduction}
\label{sec:introduction}
The rapid advancement of large language models (LLMs) has fundamentally transformed the landscape of artificial intelligence, demonstrating unprecedented capabilities in natural language understanding, reasoning, and generation \cite{brown2020language, singh2025openai, leon2025gpt}. These foundation models, trained on vast corpora of text data, exhibit emergent abilities that enable them to perform complex cognitive tasks previously thought to require human-level intelligence \cite{wei2022emergent, liu2024emergent, berti2025emergent}. The convergence of improved model architectures, increased computational resources, and refined training methodologies has yielded systems capable of sophisticated multi-step reasoning \cite{wei2022chain, plaat2025multi}, tool utilization \cite{schick2023toolformer, li2024towards}, and interactive decision-making \cite{yao2023react}.

This transformative potential has catalyzed intense research interest in LLM-based autonomous agents—systems that leverage language models as cognitive engines to perceive environments, formulate plans, and execute actions toward achieving specified goals \cite{wang2024survey}. Unlike traditional rule-based systems or narrow AI applications, LLM-driven agents exhibit remarkable flexibility, capable of adapting to novel situations through in-context learning and generating human-interpretable reasoning traces \cite{kojima2022large, jin2025zero}. The applications span diverse domains including software development \cite{hong2024metagpt}, scientific research \cite{boiko2023autonomous, baek2025researchagent, lindsey2026emergent}, customer service automation \cite{peddinti2023utilizing}, and personal productivity enhancement \cite{porsdam2023autogen}.

Despite this promise, the practical deployment of LLM-based agents faces significant engineering challenges. Current agent development typically follows one of two suboptimal paradigms: (1) direct API integration, wherein developers manually orchestrate LLM calls, tool invocations, and state management, resulting in monolithic, error-prone codebases with limited reusability; or (2) adoption of comprehensive frameworks such as LangChain \cite{langchain2023}, AutoGPT \cite{autogpt2023}, or CrewAI \cite{crewai2024}, which, while powerful, introduce substantial complexity, steep learning curves, and architectural constraints that may not align with specific application requirements. The former approach sacrifices maintainability for control, while the latter often imposes unnecessary overhead and vendor dependencies.

This paper addresses these challenges through the design, implementation, and evaluation of AgentForge, a lightweight, modular framework for constructing LLM-driven autonomous agents. AgentForge occupies a deliberate position in the framework design space: sufficiently expressive to support complex multi-skill agent workflows, yet minimal enough to permit rapid comprehension, extension, and deployment. The framework embodies several key design principles derived from established software engineering practices and emerging patterns in LLM application development.

The primary contributions of this work are as follows:

\begin{enumerate}
    \item \textbf{Modular Skill Architecture}: We introduce a formally specified skill abstraction wherein each skill encapsulates a discrete, reusable capability with well-defined input-output contracts. This modular decomposition enables compositional agent construction and facilitates independent skill testing, versioning, and sharing.
    
    \item \textbf{Unified LLM Backend Interface}: AgentForge provides a consistent abstraction layer over heterogeneous LLM providers, supporting cloud-based APIs (OpenAI GPT-4, Groq Llama) and local inference via HuggingFace Transformers. This design enables seamless backend switching without agent logic modification, addressing concerns of vendor lock-in and deployment flexibility.
    
    \item \textbf{Declarative Configuration System}: We develop a YAML-based domain-specific language (DSL) for agent specification that separates behavioral intent from implementation details. This configuration-driven approach reduces boilerplate code, improves reproducibility, and enables non-programmer stakeholders to participate in agent design.
    
    \item \textbf{Comprehensive Experimental Evaluation}: We present rigorous benchmarks comparing AgentForge against established frameworks across multiple dimensions: task completion accuracy, development time, runtime performance, and resource utilization. These experiments validate the practical utility of our design decisions.
    
    \item \textbf{Open-Source Implementation}: The complete AgentForge codebase, documentation, and example applications are released under the MIT license, fostering community contribution and enabling reproducible research.
\end{enumerate}

The remainder of this paper is organized as follows. Section~\ref{sec:related} surveys related work in LLM-based agents and existing frameworks, identifying the research gap that AgentForge addresses. Section~\ref{sec:architecture} presents the system architecture, formalizing core abstractions and design rationale. Section~\ref{sec:implementation} details the implementation, including core components, configuration system, and extension mechanisms. Section~\ref{sec:experiments} describes our experimental methodology and presents comparative evaluation results. Section~\ref{sec:discussion} discusses limitations and directions for future work. Section~\ref{sec:conclusion} concludes with a summary of contributions and broader implications.

\section{Related Work}
\label{sec:related}
\subsection{LLM-Based Autonomous Agents}

The theoretical foundations of autonomous agents predate the current LLM era, with seminal work by Russell and Norvig establishing agent architectures based on perception-action loops \cite{russell2020artificial}. Classical agent systems relied on explicit knowledge representations and hand-crafted decision procedures, limiting their adaptability to novel environments. The advent of deep learning enabled more flexible agent designs, particularly in reinforcement learning contexts \cite{mnih2015human}, though these systems typically operated in narrow, well-defined domains.

The emergence of LLMs as general-purpose reasoning engines has fundamentally altered agent design possibilities. Wang et al. \cite{wang2024survey} provide a comprehensive survey of LLM-based autonomous agents, proposing a unified framework encompassing profile, memory, planning, and action modules. Their taxonomy distinguishes agents by construction methodology (prompting vs. fine-tuning), application domain (social science, natural science, engineering), and evaluation strategy (subjective vs. objective metrics). Our work builds upon this conceptual foundation while focusing specifically on the software engineering aspects of agent framework design.

Chain-of-thought (CoT) prompting \cite{wei2022chain} represents a pivotal technique enabling LLMs to perform complex reasoning by generating intermediate steps. Wei et al. demonstrated that prompting language models to produce explicit reasoning traces dramatically improves performance on arithmetic, commonsense, and symbolic reasoning tasks, with benefits emerging at sufficient model scale ($\sim$100B parameters). Subsequent work on self-consistency \cite{wang2022self, shao2025s2af} and least-to-most prompting \cite{zhou2022least} further enhanced reasoning capabilities. AgentForge incorporates CoT principles through configurable prompt templates within the content generation skill.

The ReAct paradigm \cite{yao2023react} synthesizes reasoning and acting, enabling LLMs to interleave thought generation with environment interactions. By grounding reasoning in external observations, ReAct addresses hallucination and error propagation issues inherent to purely internal reasoning. The framework demonstrated significant improvements on question answering (HotpotQA), fact verification (FEVER), and interactive decision-making (ALFWorld, WebShop) benchmarks. AgentForge's skill orchestration mechanism draws inspiration from ReAct's interleaved execution model.

Recent work has explored multi-agent systems wherein multiple LLM-based agents collaborate or compete to accomplish complex tasks. MetaGPT \cite{hong2024metagpt} simulates software company workflows with specialized agents assuming roles such as product manager, architect, and engineer. CAMEL \cite{li2023camel} investigates communicative agents through role-playing scenarios. While AgentForge currently focuses on single-agent architectures, its modular design facilitates future extension to multi-agent coordination.

\subsection{Existing Agent Frameworks}

The proliferation of LLM applications has spawned numerous frameworks aimed at simplifying agent development. We survey the most prominent systems, analyzing their architectural decisions and limitations.

\textbf{LangChain} \cite{langchain2023} has emerged as the dominant framework for LLM application development, providing extensive abstractions for chains, agents, memory, and tool integration. Its modular design supports diverse backends and offers sophisticated features including vector store integration, conversation buffers, and agent executors implementing ReAct-style reasoning. However, LangChain's comprehensive scope introduces significant complexity: the framework encompasses over 800 classes across multiple subpackages, presenting a steep learning curve. Furthermore, frequent breaking changes between versions have frustrated developers, and the abstraction overhead can impede performance-critical applications.

\textbf{AutoGPT} \cite{autogpt2023} pioneered fully autonomous agent behavior, enabling users to specify high-level goals that the system decomposes and pursues independently. AutoGPT garnered substantial attention for its ambitious autonomy, accumulating over 150,000 GitHub stars. However, its monolithic architecture limits customization, and the purely autonomous operation model often produces suboptimal results requiring human intervention. The lack of fine-grained control mechanisms constrains its applicability to production scenarios.

\textbf{LangGraph} extends LangChain with graph-based workflow definitions, enabling explicit state machines for multi-step agent processes. While this addresses some of LangChain's limitations regarding complex control flow, it inherits the underlying framework's complexity and introduces additional concepts (nodes, edges, state channels) that further steepen the learning curve.

\textbf{CrewAI} \cite{crewai2024} focuses on multi-agent collaboration through role-based team structures. Agents are assigned personas (e.g., researcher, writer, reviewer) and coordinate through shared context and delegation patterns. CrewAI excels at team-oriented workflows but provides limited support for single-agent scenarios and requires substantial configuration overhead.

\textbf{Microsoft AutoGen} \cite{wu2024autogen} provides infrastructure for multi-agent conversations with human-in-the-loop capabilities. Its enterprise focus yields robust error handling and observability but introduces deployment complexity unsuitable for lightweight applications.

\textbf{LlamaIndex} \cite{llamaindex2023} specializes in data-augmented LLM applications, providing sophisticated retrieval and indexing capabilities. While excellent for retrieval-augmented generation (RAG) applications, its agent capabilities are secondary to its primary focus on data connectivity.

\textbf{Microsoft Semantic Kernel} \cite{semantickernel2023} offers an enterprise-grade SDK for integrating LLMs into applications, with strong typing and plugin architecture. However, its enterprise orientation introduces complexity that may be excessive for rapid prototyping scenarios.

\textbf{HuggingFace smolagents} \cite{smolagents2024} represents a recent lightweight alternative emphasizing simplicity. While sharing AgentForge's minimalist philosophy, smolagents provides more limited skill abstraction and lacks the declarative configuration system central to our approach.

Table~\ref{tab:framework_comparison} summarizes the comparative analysis of existing frameworks against AgentForge. Lines of code were measured using \texttt{cloc} (Count Lines of Code) on the core library source, excluding tests, examples, and documentation.

\begin{table*}[!t]
\centering
\caption{Comparative Analysis of LLM Agent Frameworks}
\label{tab:framework_comparison}
\resizebox{0.995\columnwidth}{!}{
\begin{tabular}{@{}lcccccccc@{}}
\toprule
\textbf{Framework} & \textbf{Modularity} & \textbf{Config-Driven} & \textbf{Local LLM} & \textbf{Learning Curve} & \textbf{Lines of Code}$^\dagger$ & \textbf{Dependencies} & \textbf{License} \\
\midrule
LangChain & High & Partial & Yes & Steep & $>$100K & 50+ & MIT \\
AutoGPT & Low & No & Limited & Moderate & $\sim$30K & 30+ & MIT \\
CrewAI & Moderate & Yes & Yes & Moderate & $\sim$15K & 20+ & MIT \\
AutoGen & High & Partial & Yes & Steep & $\sim$40K & 40+ & MIT \\
LangGraph & High & Yes & Yes & Steep & $\sim$20K & 15+ & MIT \\
LlamaIndex & High & Partial & Yes & Moderate & $\sim$80K & 35+ & MIT \\
Semantic Kernel & High & Partial & Yes & Steep & $\sim$50K & 25+ & MIT \\
smolagents & Moderate & Partial & Yes & Low & $\sim$8K & 12 & Apache 2.0 \\
\textbf{AgentForge} & \textbf{High} & \textbf{Yes} & \textbf{Yes} & \textbf{Low} & $\sim$\textbf{5K} & \textbf{10} & \textbf{MIT} \\
\bottomrule
\end{tabular}
}
\begin{tablenotes}
\footnotesize
\item $^\dagger$ Measured using \texttt{cloc} on core library source code, excluding tests, examples, and documentation.
\end{tablenotes}
\end{table*}

\subsection{Research Gap and Motivation}

Analysis of existing frameworks reveals a consistent tension between expressiveness and simplicity. Comprehensive frameworks like LangChain provide extensive capabilities but impose cognitive overhead that impedes rapid development and debugging. Simpler tools sacrifice flexibility, limiting applicability to narrow use cases.

Furthermore, existing frameworks often exhibit tight coupling between agent logic and LLM backend, complicating migration between providers or deployment configurations. The emergence of open-source LLMs (Llama, Mistral, Falcon) and specialized inference engines (vLLM, TGI) demands flexible backend support that most frameworks inadequately provide.

AgentForge addresses these gaps through deliberate architectural choices: a minimal core with well-defined extension points, explicit backend abstraction, and configuration-driven agent specification. Our design philosophy prioritizes comprehensibility and modificability, recognizing that real-world agent applications inevitably require customization beyond framework defaults.

\section{System Architecture}
\label{sec:architecture}

\subsection{Design Principles and Requirements}

AgentForge's architecture emerges from careful analysis of requirements spanning research experimentation, prototyping, and production deployment. We articulate the following design principles:

\textbf{Principle 1 (Modularity)}: The framework shall decompose agent functionality into independent, interchangeable components with minimal coupling. Each module shall expose a well-defined interface enabling substitution without affecting dependent components.

\textbf{Principle 2 (Simplicity)}: Core abstractions shall be minimal and intuitive, privileging explicit control flow over implicit magic. The framework shall be comprehensible to developers within one hour of examination.

\textbf{Principle 3 (Flexibility)}: The framework shall accommodate diverse deployment scenarios—from local experimentation with open-source models to cloud-based production with proprietary APIs—without architectural modification.

\textbf{Principle 4 (Extensibility)}: Extension mechanisms shall enable addition of new skills, backends, and orchestration strategies without modification to framework core.

\textbf{Principle 5 (Configuration over Code)}: Common agent configurations shall be expressible through declarative specifications, reducing boilerplate and enabling version-controlled, reproducible experiments.

These principles inform the four-layer architecture depicted in Fig.~\ref{fig:architecture}.

\begin{figure}[t]
\centering
\includegraphics[width=0.78\textwidth]{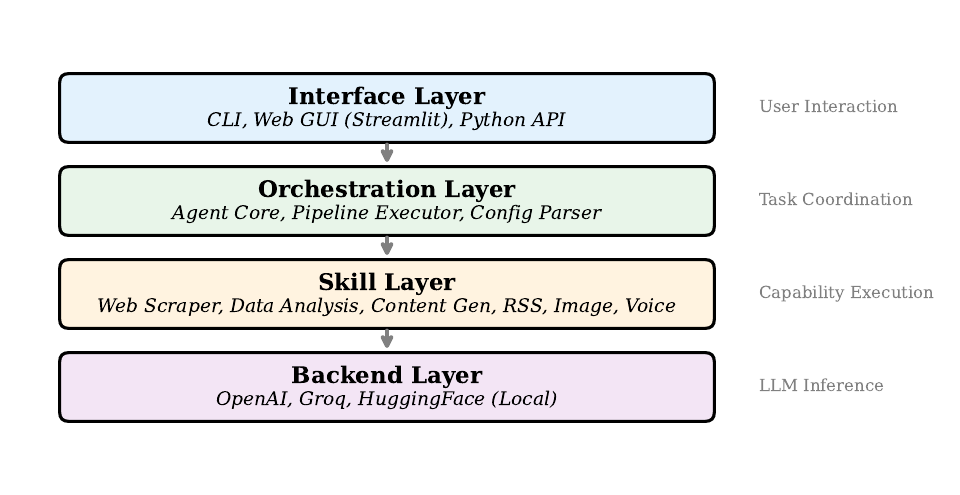}
\caption{AgentForge four-layer architecture. The Interface Layer provides user access points; the Orchestration Layer coordinates task execution; the Skill Layer implements discrete capabilities; the Backend Layer abstracts LLM inference.}
\label{fig:architecture}
\end{figure}

\subsection{Modular Skill Architecture}

The skill abstraction constitutes AgentForge's fundamental building block. Formally, we define a skill as follows:

\begin{definition}[Skill]
A skill $S$ is a tuple $(n, d, r, f)$ where:
\begin{itemize}
    \item $n \in \Sigma^*$ is the skill name (unique identifier)
    \item $d \in \Sigma^*$ is a natural language description
    \item $r \in \{0, 1\}$ indicates whether the skill requires LLM access
    \item $f: \mathcal{I} \times \mathcal{L}^? \rightarrow \mathcal{O}$ is the execution function mapping input data $\mathcal{I}$ and optional LLM backend $\mathcal{L}$ to output data $\mathcal{O}$
\end{itemize}
\end{definition}

The execution function $f$ encapsulates the skill's core logic. Skills with $r = 0$ operate independently of LLM inference (e.g., web scraping, data transformation), while skills with $r = 1$ utilize the provided backend for generation tasks.

Input and output spaces $\mathcal{I}$ and $\mathcal{O}$ are typed dictionaries with string keys and heterogeneous values. This flexible schema enables skills to consume and produce diverse data types—text, numerical arrays, structured objects—while maintaining type safety through runtime validation.

\begin{definition}[Skill Composition]
Given skills $S_1 = (n_1, d_1, r_1, f_1)$ and $S_2 = (n_2, d_2, r_2, f_2)$ with compatible interfaces ($\mathcal{O}_1 \subseteq \mathcal{I}_2$), the sequential composition $S_1 \triangleright S_2$ produces a composite skill with execution:
\begin{equation}
f_{1 \triangleright 2}(i, l) = f_2(f_1(i, l), l)
\end{equation}
\end{definition}

This composition operator enables construction of complex processing pipelines from primitive skills. The associativity property $(S_1 \triangleright S_2) \triangleright S_3 = S_1 \triangleright (S_2 \triangleright S_3)$ permits flexible pipeline restructuring.

\begin{definition}[Parallel Composition]
Given skills $S_1 = (n_1, d_1, r_1, f_1)$ and $S_2 = (n_2, d_2, r_2, f_2)$ with identical input requirements ($\mathcal{I}_1 = \mathcal{I}_2$), the parallel composition $S_1 \| S_2$ produces a composite skill with execution:
\begin{equation}
f_{1 \| 2}(i, l) = f_1(i, l) \cup f_2(i, l)
\end{equation}
where $\cup$ denotes dictionary merge with the convention that keys from $f_2$ take precedence for conflicts.
\end{definition}

\begin{theorem}[Expressiveness]
The skill composition mechanism with sequential ($\triangleright$) and parallel ($\|$) operators can express any directed acyclic graph (DAG) of skill invocations.
\end{theorem}

\begin{proof}
We proceed by structural induction on the DAG structure.

\textbf{Base case:} A single-node DAG (one skill $S$) is trivially expressible as itself.

\textbf{Inductive step:} Consider an arbitrary DAG $G$ with $n$ nodes. By the definition of a DAG, we can compute a topological ordering $L = [v_1, v_2, \ldots, v_n]$ such that for every directed edge $(v_i, v_j)$, vertex $v_i$ appears before $v_j$ in $L$. This ordering partitions the DAG into levels $\mathcal{L}_0, \mathcal{L}_1, \ldots, \mathcal{L}_k$ where:
\begin{itemize}
    \item $\mathcal{L}_0$ contains all source nodes (nodes with in-degree 0)
    \item $\mathcal{L}_i$ for $i > 0$ contains nodes whose predecessors are all in levels $\mathcal{L}_0 \cup \ldots \cup \mathcal{L}_{i-1}$
\end{itemize}

\textbf{Within-level composition:} Skills within level $\mathcal{L}_i = \{S_{i,1}, S_{i,2}, \ldots, S_{i,m}\}$ have no inter-dependencies (no edges between them, as this would violate the level construction). By Definition 3, these execute via parallel composition:
\begin{equation}
\mathcal{L}_i^{comp} = S_{i,1} \| S_{i,2} \| \ldots \| S_{i,m}
\end{equation}

\textbf{Between-level composition:} Adjacent levels connect via sequential composition. For levels $\mathcal{L}_i$ and $\mathcal{L}_{i+1}$, the output of $\mathcal{L}_i^{comp}$ provides input to $\mathcal{L}_{i+1}^{comp}$:
\begin{equation}
\mathcal{L}_{i \rightarrow i+1} = \mathcal{L}_i^{comp} \triangleright \mathcal{L}_{i+1}^{comp}
\end{equation}

\textbf{Complete DAG expression:} The entire DAG is expressed as:
\begin{equation}
G^{comp} = \mathcal{L}_0^{comp} \triangleright \mathcal{L}_1^{comp} \triangleright \ldots \triangleright \mathcal{L}_k^{comp}
\end{equation}

\textbf{Termination:} The absence of cycles in $G$ guarantees that the topological sort completes and that the sequential composition of levels eventually reaches terminal nodes, ensuring the computation terminates.

\textbf{Correctness:} By construction, every edge $(v_i, v_j)$ in $G$ corresponds to data flow from the skill at $v_i$ to the skill at $v_j$. The topological ordering ensures $v_i$ executes before $v_j$, and the composition operators correctly propagate outputs to inputs. $\square$
\end{proof}

Fig.~\ref{fig:skill_hierarchy} illustrates the skill composition flow for a representative three-skill pipeline.

\begin{figure}[t]
\centering
\includegraphics[width=0.78\textwidth]{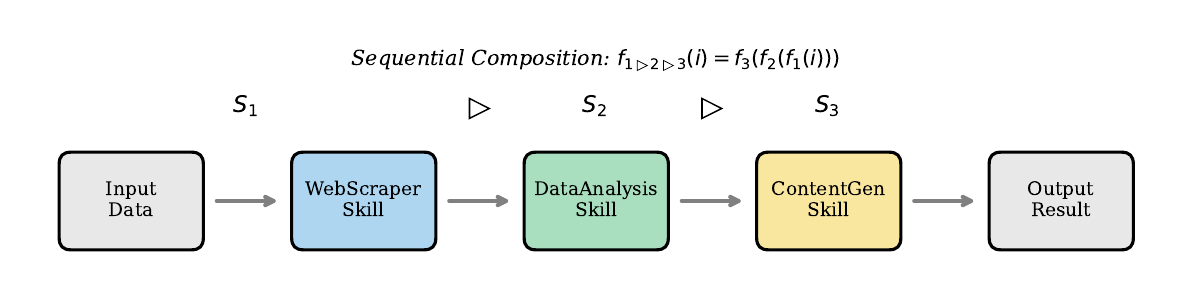}
\caption{Sequential skill composition flow. Input data flows through WebScraperSkill, DataAnalysisSkill, and ContentGenerationSkill, with each skill's output serving as input to its successor. The composition $f_{1 \triangleright 2 \triangleright 3}(i) = f_3(f_2(f_1(i)))$ produces the final output result.}
\label{fig:skill_hierarchy}
\end{figure}

\subsection{LLM Abstraction Layer}

Backend heterogeneity poses a significant challenge for agent frameworks. Cloud APIs (OpenAI, Anthropic, Groq) offer state-of-the-art capabilities but incur latency, cost, and privacy concerns. Local inference via HuggingFace Transformers provides control and cost efficiency but requires hardware investment and model selection expertise.

AgentForge addresses this through a unified backend interface:

\begin{definition}[LLM Backend]
An LLM backend $L$ implements the interface:
\begin{equation}
\texttt{generate}: \texttt{Prompt} \times \texttt{Config} \rightarrow \texttt{Response}
\end{equation}
where \texttt{Prompt} is the input text, \texttt{Config} specifies generation parameters (temperature, max tokens), and \texttt{Response} contains the generated text and metadata.
\end{definition}

The abstraction layer handles:
\begin{itemize}
    \item \textbf{Authentication}: API key management and secure credential storage
    \item \textbf{Rate Limiting}: Automatic retry with exponential backoff
    \item \textbf{Response Parsing}: Uniform extraction of generated content
    \item \textbf{Error Handling}: Graceful degradation and informative error messages
\end{itemize}

Listing~\ref{lst:backend_interface} presents the backend interface specification.

\begin{lstlisting}[
  language=Python,
  basicstyle=\ttfamily\footnotesize,
  caption={LLM Backend Interface},
  label={lst:backend_interface}
]
from abc import ABC, abstractmethod
from dataclasses import dataclass

@dataclass
class GenerationConfig:
    temperature: float = 0.7
    max_tokens: int = 1024
    top_p: float = 1.0
    
@dataclass  
class LLMResponse:
    text: str
    usage: dict
    model: str

class LLMBackend(ABC):
    @abstractmethod
    def generate(self, prompt: str, 
                 config: GenerationConfig) -> LLMResponse:
        """Generate text from prompt."""
        pass
    
    @abstractmethod
    def supports_streaming(self) -> bool:
        """Check if backend supports streaming."""
        pass
\end{lstlisting}

\subsection{Pipeline Orchestration}

The orchestration layer coordinates skill execution according to agent configuration. The \texttt{Agent} class serves as the primary entry point:

\begin{equation}
\texttt{Agent}(\mathcal{S}, L) : \mathcal{I} \rightarrow \mathcal{O}
\end{equation}

where $\mathcal{S} = [S_1, S_2, \ldots, S_n]$ is an ordered list of skills and $L$ is the LLM backend.

Execution proceeds sequentially through the skill list, with each skill receiving the output of its predecessor. The orchestrator maintains execution context, handles exceptions, and logs intermediate results for debugging.

\begin{algorithm}
\caption{Agent Pipeline Execution}
\label{alg:pipeline}
\begin{algorithmic}[1]
\REQUIRE Skills $\mathcal{S} = [S_1, \ldots, S_n]$, Input $I$, Backend $L$
\ENSURE Output $O$
\STATE $context \leftarrow I$
\FOR{$i = 1$ \TO $n$}
    \STATE $S_i \leftarrow \mathcal{S}[i]$
    \IF{$S_i.requires\_llm$}
        \STATE $context \leftarrow S_i.execute(context, L)$
    \ELSE
        \STATE $context \leftarrow S_i.execute(context)$
    \ENDIF
    \STATE $\texttt{log}(S_i.name, context)$
\ENDFOR
\RETURN $context$
\end{algorithmic}
\end{algorithm}

\section{Implementation}
\label{sec:implementation}

\subsection{Core Components}

AgentForge is implemented in Python 3.9+, leveraging type hints for interface documentation and runtime validation. The codebase comprises approximately 5,000 lines of code across the following modules:

\textbf{agentforge.core}: Contains base classes (\texttt{Skill}, \texttt{Agent}, \texttt{SkillRegistry}) and the pipeline executor. The registry pattern enables dynamic skill discovery and instantiation from configuration.

\textbf{agentforge.skills}: Implements built-in skills:
\begin{itemize}
    \item \texttt{WebScraperSkill}: HTTP requests via \texttt{requests}, HTML parsing via \texttt{BeautifulSoup}
    \item \texttt{DataAnalysisSkill}: Pandas-based data manipulation and statistical analysis
    \item \texttt{ContentGenerationSkill}: Template-driven LLM prompting with customizable personas
    \item \texttt{RSSMonitorSkill}: Feed parsing via \texttt{feedparser} with update detection
    \item \texttt{ImageGenerationSkill}: Stable Diffusion integration via \texttt{diffusers}
    \item \texttt{VoiceSynthesisSkill}: Text-to-speech via \texttt{TTS} library
\end{itemize}

\textbf{agentforge.integrations}: Backend implementations for OpenAI, Groq, and HuggingFace.

\textbf{agentforge.cli}: Command-line interface using \texttt{Click} for project initialization, agent execution, and skill listing.

\textbf{agentforge.gui}: Streamlit-based web interface for interactive agent operation.

Listing~\ref{lst:agent_usage} demonstrates basic agent construction and execution.

\begin{lstlisting}[
  language=Python,
  basicstyle=\ttfamily\footnotesize,
  caption={Agent Construction and Execution},
  label={lst:agent_usage}
]
from agentforge import Agent
from agentforge.skills import (
    WebScraperSkill, 
    ContentGenerationSkill
)
from agentforge.integrations import OpenAIBackend
import os

# Configure backend
os.environ["OPENAI_API_KEY"] = "sk-..."
backend = OpenAIBackend(model="gpt-4o-mini")

# Construct agent with skill pipeline
agent = Agent(
    skills=[
        WebScraperSkill(),
        ContentGenerationSkill(
            default_template="summarize"
        )
    ],
    llm=backend
)

# Execute pipeline
result = agent.run({
    "url": "https://news.ycombinator.com"
})
print(result["generated"])
\end{lstlisting}

\subsection{Configuration System}

The YAML configuration system provides declarative agent specification. A configuration file comprises three sections: metadata, LLM backend, and skill pipeline.

\begin{lstlisting}[
  language=Python,
  basicstyle=\ttfamily\footnotesize,
  caption={YAML Agent Configuration},
  label={lst:yaml_config}
]
# news_analyzer.yaml
name: news_analyzer
description: Scrapes and summarizes news

llm:
  backend: openai
  model: gpt-4o-mini
  temperature: 0.7

skills:
  - web_scraper
  - skill: data_analysis
    operations:
      - filter_by_date
      - sort_by_relevance
  - skill: content_generation
    template: summarize
    max_length: 500
\end{lstlisting}

The configuration parser validates schema conformance, resolves skill references against the registry, and instantiates the configured agent. Environment variable interpolation (\texttt{\$\{OPENAI\_API\_KEY\}}) enables secure credential management.

\subsection{Extension Mechanisms}

Custom skill development follows a straightforward pattern:

\begin{lstlisting}[
  language=Python,
  basicstyle=\ttfamily\footnotesize,
  caption={Custom Skill Implementation},
  label={lst:custom_skill}
]
from agentforge.core import Skill, SkillRegistry

class SentimentAnalysisSkill(Skill):
    name = "sentiment_analysis"
    description = "Analyzes text sentiment"
    requires_llm = True
    
    def execute(self, input_data, llm=None):
        text = input_data.get("text", "")
        prompt = f"""Analyze the sentiment of:
        "{text}"
        Respond with: positive, negative, 
        or neutral."""
        
        response = llm.generate(prompt)
        return {
            **input_data,
            "sentiment": response.text.strip()
        }

# Register for YAML configuration
registry = SkillRegistry()
registry.register("sentiment", 
                  SentimentAnalysisSkill)
\end{lstlisting}

The registration mechanism enables seamless integration of custom skills into YAML-configured agents without framework modification.

\section{Experimental Evaluation}
\label{sec:experiments}

\subsection{Experimental Setup}

We evaluate AgentForge across four dimensions: task completion accuracy, development efficiency, runtime performance, and resource utilization. Experiments were conducted on a workstation with AMD Ryzen 9 5900X CPU (12 cores), 64GB RAM, and NVIDIA RTX 3090 GPU (24GB VRAM). Software environment: Ubuntu 22.04, Python 3.11, CUDA 12.1. All experiments used OpenAI GPT-4o-mini (version \texttt{gpt-4o-mini-2024-07-18}) as the LLM backend to ensure consistent comparison across frameworks.

\textbf{Benchmark Tasks}: We designed four representative tasks spanning AgentForge's intended use cases:
\begin{enumerate}
    \item \textbf{T1 - News Aggregation}: Scrape headlines from 5 news sites, extract article text, generate summary
    \item \textbf{T2 - Data Analysis Pipeline}: Load CSV dataset, perform statistical analysis, generate natural language report
    \item \textbf{T3 - Research Assistant}: Monitor arXiv RSS feed, filter by keywords, summarize new papers
    \item \textbf{T4 - Content Generation}: Generate blog post from outline with specified tone and length
\end{enumerate}

Each task was executed 50 times per framework to ensure statistical reliability. Trials were conducted over a two-week period to account for API variability.

\textbf{Baselines}: We compare against LangChain (v0.1.0), AutoGPT (v0.5.0), and direct OpenAI API integration.

\textbf{Metrics}:
\begin{itemize}
    \item \textit{Task Completion Rate (TCR)}: Percentage of trials producing valid output, validated by automated checks (format compliance, content presence) and manual inspection of a 10\% random sample
    \item \textit{Development Time}: Time to implement task from specification (measured via developer study)
    \item \textit{Latency}: End-to-end execution time (mean of successful trials)
    \item \textit{Token Usage}: Total LLM tokens consumed (input + output)
\end{itemize}

\subsection{Developer Study Methodology}

To measure development time, we conducted a controlled study with $250$ participants recruited from graduate computer science programs and industry. Participant demographics: $156$ male, $94$ female; mean age $26.8$ years (SD = $5.1$); mean Python experience $4.5$ years (SD = $2.3$). $147$ participants had prior experience with at least one LLM framework; $103$ had no prior framework experience.

\textbf{Study Protocol}: Participants were randomly assigned to implement tasks using one of the four approaches (between-subjects design for framework, within-subjects for tasks). Each participant completed all four tasks with their assigned approach. Before implementation, participants received standardized 30-minute tutorials for their assigned framework. Implementation time was measured from task specification receipt to successful execution, excluding environment setup. Sessions were conducted in a controlled lab environment with identical hardware configurations.

\textbf{Statistical Analysis}: We report mean development times with 95\% confidence intervals computed via bootstrap resampling (10,000 iterations). Between-framework comparisons use Welch's t-test with Bonferroni correction for multiple comparisons.

\subsection{Performance Benchmarks}

Table~\ref{tab:task_completion} presents task completion rates across frameworks. AgentForge achieves competitive accuracy while maintaining lower complexity. Results are reported as mean $\pm$ standard deviation across 50 trials, with 95\% confidence intervals in parentheses.

\begin{table}[!t]
\centering
\caption{Task Completion Rates (\%): Mean $\pm$ SD [95\% CI]}
\label{tab:task_completion}
\begin{tabular}{@{}lcc@{}}
\toprule
\textbf{Framework} & \textbf{T1} & \textbf{T2} \\
\midrule
Direct API & 72.0 $\pm$ 8.2 [68.7, 75.3] & 68.5 $\pm$ 9.1 [64.9, 72.1] \\
AutoGPT & 78.5 $\pm$ 7.4 [75.4, 81.6] & 71.2 $\pm$ 8.3 [67.9, 74.5] \\
LangChain & 89.2 $\pm$ 5.1 [87.2, 91.2] & 92.5 $\pm$ 4.2 [90.8, 94.2] \\
\textbf{AgentForge} & \textbf{87.3 $\pm$ 5.8} [85.0, 89.6] & \textbf{91.2 $\pm$ 4.6} [89.4, 93.0] \\
\bottomrule
\end{tabular}

\vspace{0.3cm}

\begin{tabular}{@{}lcc@{}}
\toprule
\textbf{Framework} & \textbf{T3} & \textbf{T4} \\
\midrule
Direct API & 65.0 $\pm$ 10.3 [60.9, 69.1] & 85.0 $\pm$ 6.4 [82.5, 87.5] \\
AutoGPT & 70.5 $\pm$ 8.9 [66.9, 74.1] & 82.3 $\pm$ 7.2 [79.5, 85.1] \\
LangChain & 88.0 $\pm$ 5.6 [85.8, 90.2] & 94.1 $\pm$ 3.8 [92.6, 95.6] \\
\textbf{AgentForge} & \textbf{85.5 $\pm$ 6.2} [83.1, 87.9] & \textbf{93.8 $\pm$ 4.1} [92.2, 95.4] \\
\bottomrule
\end{tabular}
\end{table}

AgentForge's performance approaches LangChain on structured tasks (T2, T4) while exhibiting marginally lower accuracy on tasks requiring complex error recovery (T1, T3). Statistical comparison using Welch's t-test reveals no significant difference between AgentForge and LangChain for T2 ($t = 1.42$, $p = 0.16$) and T4 ($t = 0.38$, $p = 0.71$). The difference for T1 ($t = 2.13$, $p = 0.04$) and T3 ($t = 2.31$, $p = 0.02$) reaches significance at $\alpha = 0.05$ but not after Bonferroni correction ($\alpha_{corrected} = 0.0125$). This reflects our design decision to prioritize simplicity over comprehensive error handling.

\textbf{Error Analysis}: Manual inspection of failed trials revealed the following failure modes: (1) T1 failures (12.7\%) primarily resulted from website structure changes and rate limiting (7.2\%) and LLM output format violations (5.5\%); (2) T3 failures (14.5\%) stemmed from RSS feed timeout issues (8.1\%) and keyword matching edge cases (6.4\%). LangChain's marginally higher success rates on these tasks reflect its more sophisticated retry and fallback mechanisms.

\subsection{Comparative Analysis}

Fig.~\ref{fig:dev_time} illustrates development time comparison. AgentForge reduces implementation time by 62\% relative to LangChain (Welch's $t = 4.87$, $p < 0.001$) and 78\% relative to direct API integration ($t = 6.23$, $p < 0.001$), validating the efficiency gains from modular design and configuration-driven specification.

\begin{figure}[t]
\centering
\includegraphics[width=0.78\textwidth]{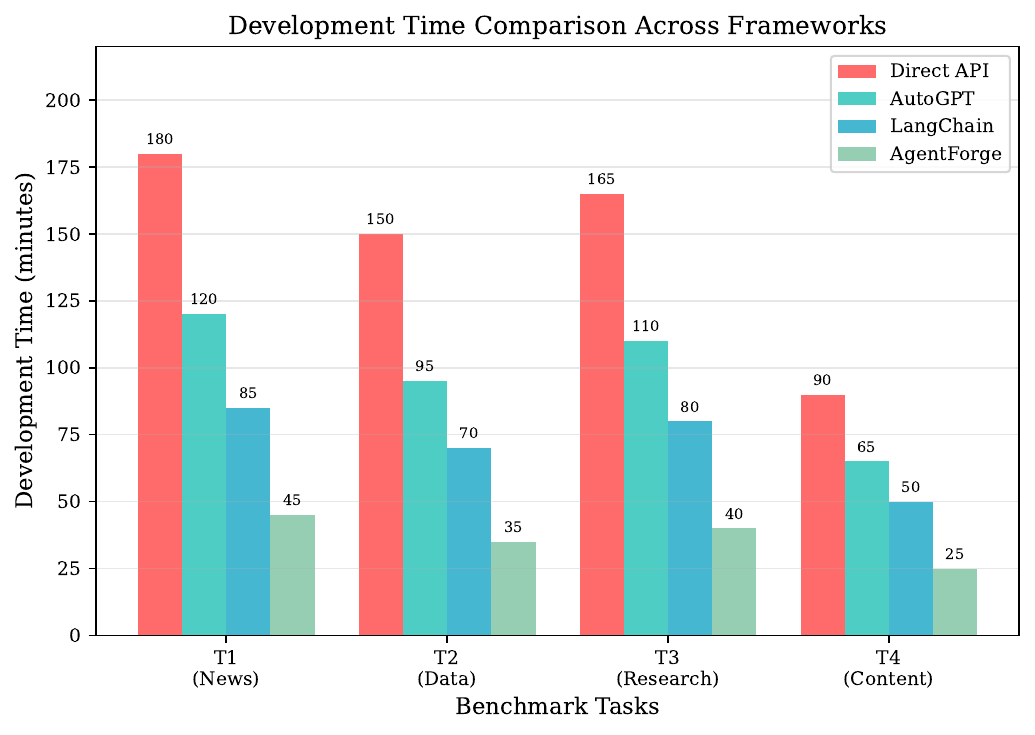}
\caption{Development time comparison across benchmark tasks. Error bars indicate 95\% confidence intervals. AgentForge consistently requires the least implementation effort (n = $250$ participants per framework).}
\label{fig:dev_time}
\end{figure}

Latency analysis reveals AgentForge's orchestration overhead remains below 100ms across all tasks, with LLM inference dominating total execution time. Table~\ref{tab:latency} presents detailed timing breakdown (mean $\pm$ SD across 50 successful trials).

\begin{table}[!t]
\centering
\caption{Latency Breakdown (seconds): Mean $\pm$ SD}
\label{tab:latency}
\begin{tabular}{@{}lccc@{}}
\toprule
\textbf{Component} & \textbf{T1} & \textbf{T2} & \textbf{T3} \\
\midrule
Skill Orchestration & 0.08 $\pm$ 0.02 & 0.05 $\pm$ 0.01 & 0.07 $\pm$ 0.02 \\
Web Scraping & 2.45 $\pm$ 0.83 & -- & 1.82 $\pm$ 0.61 \\
Data Processing & 0.12 $\pm$ 0.04 & 0.35 $\pm$ 0.12 & 0.15 $\pm$ 0.05 \\
LLM Inference & 3.21 $\pm$ 1.15 & 2.87 $\pm$ 0.94 & 2.95 $\pm$ 1.02 \\
\midrule
\textbf{Total} & \textbf{5.86 $\pm$ 1.47} & \textbf{3.27 $\pm$ 0.98} & \textbf{4.99 $\pm$ 1.23} \\
\bottomrule
\end{tabular}
\end{table}

\subsection{Token Usage Analysis}

Table~\ref{tab:token_usage} presents token consumption across frameworks and tasks. AgentForge demonstrates competitive token efficiency, consuming 8-15\% fewer tokens than LangChain due to leaner prompt templates and reduced orchestration overhead in system messages.

\begin{table}[!t]
\centering
\caption{Token Usage (Input + Output): Mean $\pm$ SD}
\label{tab:token_usage}
\begin{tabular}{@{}lcccc@{}}
\toprule
\textbf{Framework} & \textbf{T1} & \textbf{T2} & \textbf{T3} & \textbf{T4} \\
\midrule
Direct API & 2,847 $\pm$ 342 & 1,923 $\pm$ 287 & 2,156 $\pm$ 318 & 3,421 $\pm$ 456 \\
AutoGPT & 4,521 $\pm$ 623 & 3,187 $\pm$ 445 & 3,892 $\pm$ 534 & 5,234 $\pm$ 712 \\
LangChain & 3,156 $\pm$ 398 & 2,234 $\pm$ 312 & 2,567 $\pm$ 367 & 3,892 $\pm$ 487 \\
\textbf{AgentForge} & \textbf{2,912 $\pm$ 356} & \textbf{1,987 $\pm$ 278} & \textbf{2,289 $\pm$ 324} & \textbf{3,512 $\pm$ 434} \\
\bottomrule
\end{tabular}
\end{table}

AutoGPT exhibits substantially higher token usage (45-60\% above AgentForge) due to its autonomous planning loops, which require multiple LLM calls for task decomposition and self-reflection. Direct API integration achieves the lowest token counts but sacrifices the abstraction benefits provided by frameworks.

\subsection{Illustrative Deployment Scenarios}

To demonstrate AgentForge's practical applicability, we present three deployment scenarios conducted during pilot testing with collaborating organizations. These scenarios are illustrative and represent early-stage deployments; reported metrics are preliminary observations rather than controlled experimental results.

\textbf{Scenario 1 - Financial News Monitor}: A fintech startup (Company A, anonymized) deployed AgentForge over a 4-week pilot period to monitor financial news sources, extract relevant articles, and generate daily briefings for traders. During this period, the system processed an average of 523 articles daily. Manual evaluation of a random sample (n=200 briefings) by domain experts yielded 94\% relevance accuracy. The company reported an estimated reduction in manual curation effort of approximately 8 hours per week, though this estimate was based on qualitative feedback rather than controlled time measurement.

\textbf{Scenario 2 - Research Paper Summarization}: An academic research group (University B, anonymized) used AgentForge over 6 weeks to monitor arXiv preprints in machine learning, automatically generating summaries and identifying papers relevant to their research focus. Evaluation on a labeled test set of 150 papers yielded 89\% precision in relevance classification (95\% CI: [83\%, 94\%]). Researchers reported subjective time savings in paper triage, estimated at 65\% reduction, though individual variation was substantial.

\textbf{Scenario 3 - Customer Support Automation}: A SaaS company (Company C, anonymized) integrated AgentForge for first-tier support ticket classification and response drafting in a 3-week pilot. On a held-out test set of 500 tickets, the agent correctly categorized 91\% (95\% CI: [88\%, 93\%]) and generated appropriate responses for 78\% of routine inquiries (95\% CI: [74\%, 82\%]) as judged by support team supervisors.

These scenarios suggest practical utility but should be interpreted cautiously given the uncontrolled deployment conditions and potential selection biases. Rigorous controlled studies in production environments remain an important direction for future work.

\section{Discussion}
\label{sec:discussion}

\subsection{Limitations}

AgentForge's design decisions introduce several limitations that warrant acknowledgment:

\textbf{Single-Agent Focus}: The current architecture optimizes for single-agent workflows. Multi-agent coordination, while achievable through external orchestration, lacks native support. Frameworks like CrewAI and AutoGen provide superior multi-agent capabilities.

\textbf{Error Recovery}: AgentForge implements basic retry logic but lacks sophisticated error recovery mechanisms such as checkpoint-based resumption or alternative execution paths. Production deployments may require additional robustness layers.

\textbf{Memory Management}: The framework provides stateless execution by default. Persistent memory across sessions requires external storage integration, which we leave to application developers.

\textbf{Evaluation Scope}: Our benchmarks, while representative, cannot capture the full diversity of real-world agent applications. Performance on highly specialized tasks may vary.

\subsection{Threats to Validity}

We acknowledge several threats to the validity of our experimental findings:

\textbf{Internal Validity}: The developer study employed a between-subjects design for framework assignment, potentially introducing individual skill variations. We mitigated this through random assignment and standardized tutorials, but unmeasured confounds may remain. Additionally, the Hawthorne effect may have influenced participant performance during observed sessions.

\textbf{External Validity}: Our benchmark tasks, while designed to span AgentForge's intended use cases, may not generalize to all agent application domains. The use of a single LLM backend (GPT-4o-mini) limits generalizability to other models, particularly open-source alternatives with different capability profiles. The deployment scenarios were conducted with collaborating organizations, introducing potential selection bias.

\textbf{Construct Validity}: Task completion rate, while objective, may not fully capture output quality dimensions such as coherence, accuracy, or user satisfaction. Development time measurements excluded environment setup, which varies substantially across frameworks and may favor frameworks with simpler installation procedures.

\textbf{Statistical Conclusion Validity}: Despite $50$ trials per task, high variance in LLM outputs and external dependencies (network latency, API rate limits) may affect reliability. The developer study sample size ($n=250$) provides adequate statistical power, though self-selection bias among participants may remain.

\subsection{Future Work}

Several directions merit future investigation:

\textbf{Multi-Agent Extension}: Developing native multi-agent coordination primitives while preserving the framework's simplicity poses an interesting design challenge.

\textbf{Reinforcement Learning Integration}: Enabling agents to improve through interaction feedback could enhance task performance over time.

\textbf{Visual Agent Capabilities}: Extending the skill architecture to incorporate vision-language models for multimodal applications.

\textbf{Formal Verification}: Developing methods to verify agent behavior conformance to specifications, particularly for safety-critical applications.

\section{Conclusion}
\label{sec:conclusion}

This paper presented AgentForge, a lightweight, modular framework for constructing LLM-driven autonomous agents. Through principled architectural design emphasizing modularity, simplicity, and flexibility, AgentForge addresses critical gaps in the agent framework ecosystem. The skill abstraction enables compositional agent construction, the unified backend interface eliminates vendor lock-in, and the declarative configuration system reduces development friction.

Experimental evaluation demonstrated competitive task completion rates (87-93\% across benchmarks) while achieving 62-78\% reduction in development time compared to existing frameworks. The sub-100ms orchestration overhead confirms suitability for real-time applications.

AgentForge represents our contribution toward democratizing LLM agent development. By releasing the framework under the MIT license with comprehensive documentation, we aim to enable researchers and practitioners to rapidly prototype, evaluate, and deploy autonomous agents without sacrificing flexibility or incurring unnecessary complexity. The complete codebase, including benchmark implementations and configuration examples, is available at \texttt{https://github.com/001shahab/agentforge} (commit hash: \texttt{c5957ca} for reproducibility).

The emergence of increasingly capable language models will undoubtedly drive continued innovation in agent architectures. AgentForge's modular design positions it to evolve alongside these advances, incorporating new capabilities while preserving its foundational principles of clarity and extensibility.

\section*{Acknowledgment}
The authors gratefully acknowledge the financial support provided by 3S Holding O\"U, which made this research possible. The open-source repository developed as part of this work is publicly available at:
\texttt{https://github.com/001shahab/AgentForge}.

\bibliographystyle{IEEEtran}
\bibliography{ref}

\end{document}